%% file: UAI32.tex
\documentclass{article}
\usepackage{proceed}
\include{mydefs}

\def\lauc{L_\text{AUC}}
\def\comment#1{}

\title{A Univariate Bound of Area Under ROC}
\author{Siwei Lyu $~~~~~~~~~~~~~~~~~~~$ Yiming Ying \\ Computer Science Department $~~$ Mathematics Department \\University at Albany, State University at New York, USA \\ \{slyu,yying\}@albany.edu}
\begin{document}
\maketitle

\begin{abstract}
Area under ROC (AUC) is an important metric for binary classification and bipartite ranking problems. However, it is difficult to directly optimize AUC as a learning objective, so most existing algorithms are based on optimizing a surrogate loss to AUC. One significant drawback of these surrogate losses is that they require pairwise comparisons among training data, which leads to slow running time and increasing local storage for online learning. In this work, we describe a new surrogate loss based on a reformulation of AUC risk, which does not require pairwise comparison but rankings of the predictions. We further show that the ranking operation can be avoided, and the learning objective obtained based on this surrogate enjoys linear complexity in time and storage. We perform experiments to demonstrate the effectiveness of the online and batch algorithms for AUC optimization based on the proposed surrogate loss. 
\end{abstract}

\section{Introduction}

The {\em area under receiver operating characteristics curves} (AUC) is a useful quantitative metric for assessing the performance of binary classification and bipartite ranking algorithms \cite{Cortes:2003cv,W.-Kotlowski:2011ei}. However, there are two factors make AUC difficult to be used directly as a learning objective to train classification or ranking algorithms. The foremost is due to the discontinuous indicator function in the definition of the AUC (c.f. Eq.\eqref{e:auc}), which makes direct minimization of the AUC in general an NP-hard problem \cite{hand_till_ml01}.   As such, most existing AUC learning algorithm replace the indicator function with surrogates that are continuous and convex upper-bounds of the AUC.  The second issue with the AUC is the requirement of pairwise comparison between all positive and negative examples in training data. This leads to algorithms with a running time complexity that is quadratic in the number of training data, and a space complexity that is linear of the training data. For batch algorithms, this means slow running time as we need to compare all pairs of positive/negative examples, and for online learning, this means ever increasing local storage as we need to store all previously seen data for the pairwise comparisons. Both are undesirable when applying these algorithms to large-scale datasets. 

In this work, we describe a new surrogate loss to AUC that has a linear time complexity and constant space complexity. This new loss is based on an equivalent formulation of AUC based on ranking the prediction scores, which obviates pairwise comparisons. We further show that the ranking operation can be replaced with an equivalent optimization problem, and the learning objective affords a simple form that has a bounding relation with AUC. Furthermore,  we show that the new loss has a close relation with the SVM learning objective, which sheds light on the previous observations of the effectiveness of the SVM on optimizing AUC \cite{Brefeld:2005dd,Joachims:2005rv,Caruana:2004nr,Rakotomamonjy:2004oc}.  The new surrogate loss leads naturally to an online AUC optimization method with simple (projected) stochastic sub-gradient steps.  Experimental evaluations on several standard benchmark datasets show that learning objective formed from this new loss achieves performance in par with other widely used AUC surrogates, with a significant reduction in running time and storage requirement. 

\section{Definitions}

To facilitate subsequent description, we first review the definition of AUC in the context of binary classification. Assume we are given a set of data $\{(\x_i, y_i)\}_{i=1}^N$, with $y_i \in \{-1, +1\}$ and $\x_i \in \R^d$. We denote $\I^+ = \{i|y_i = +1\}$ and $\I^- = \{i|y_i = -1\}$ as the sets of indices of positive and negative examples, respectively, with $N^+ = |\I^+|$ and $N^- = |\I^-|$, and $N^+ + N^- = N$. Define $\mathbf{I}$ as the {\em indicator function}: $\mathbf{I}_a = 1$ if $a$ is true and 0 otherwise.  A parametric binary classifier $c_{\w,\theta}: \R^d \mapsto \{-1,+1\}$, constructed as
~\vspace{-.4em} \[ 
c_{\w,\theta}(\x) = 2\mathbf{I}_{f_\w(\x) \ge \theta}-1 = \sgn(f_\w(\x) - \theta),
\]
maps an example to the class label, where $f_\w: \R^d \mapsto \R$ (with $\w \in \R^m$ being the parameter) is the prediction function and $\theta \in \R$ is the classification threshold. We denote $c_i = f_\w(\x_i)$ as the {\em prediction score} of the i$^\text{th}$ example ($i = 1,\cdots,N$). For simplicity, we assume there are no ties in the prediction scores, \ie, $c_i \not = c_j$ for $i \not = j$, though this condition will be relaxed later.

Given a threshold $\theta$, negative examples with prediction scores greater than $\theta$ are false positives, and the false positive rate is given by $\tau_{FP} = |\mathbf{I}_{c_i > \theta \land i \in \I^- }|/N^-$. Correspondingly, positive examples with prediction scores greater or equal to $\theta$ are true positives, and the true positive rate is given by $\tau_{TP} = |\mathbf{I}_{c_i \ge \theta \land i \in \I^+ }|/N^+$. Then the receiver operation curve (ROC) is defined as the curve formed by the pair $(\tau_{FP},\tau_{TP})$ with $\theta \in (-\infty,\infty)$. With this definition, ROC is a curve confined to $[0,1]\times[0,1]$ and connecting $(0,0)$ to $(1,1)$.  AUC then corresponds to the area enclosed by the ROC curve of the classifier. 

It is more conveniently computed in closed form using the Wilcoxon-Mann-Whitney (WMW) statistic \cite{Hanley:1982dz}, as $A = {1 \over N^+N^-} \sum_{i \in \I^+} \sum_{j \in \I^-} \mathbf{I}_{c_i > c_j}$.  In this work, we use the {\em AUC risk}, which is defined as 
\begin{equation}
\lauc = 1-A =  {1 \over N^+N^-} \sum_{i \in \I^+} \sum_{j \in \I^-} \mathbf{I}_{c_i < c_j}.
\label{e:auc}
\end{equation}
Note that $\lauc$ takes values in $[0, 1]$ and corresponds to the fraction of pairs of positive and negative predictions that are ranked incorrectly, \ie, a positive example with lower prediction score than a negative example, so $\lauc = 0$ indicates perfect classification/ranking. In addition, $\lauc$ is independent of threshold $\theta$, and only concerns with the overall performance of the predictor $f_\w$.  Hence, we aim to learn a prediction function $f_\w$ that minimizes $\lauc$, from which we can choose $\theta$ to construct classifier $c_{\w,\theta}(\x)$. 

\section{Related Works}

Most existing works for either batch or online algorithms for AUC optimization (\eg, \cite{P.-Zhao:2011wb,W.-Gao:2013zo}) minimize surrogates to the true AUC-risk, which are usually in the form of convex upper-bounds to the indicator function in Eq.\eqref{e:auc}. Specifically, denoting the prediction scores for $\x_i$ and $\x_j$ as $c_i$ and $c_j$, respectively, widely used surrogate loss include 
\begin{enumerate} \itemsep 0em
\item the hinge function \cite{W.-Gao:2013zo}, $\ell_h(c_i,c_j) = [1- (c_i - c_j)]_+$, where $\left[a\right]_{+} = \max\{0,a\}$ is the hinge function, 
\item the squared hinge loss \cite{ying_etal_nips16,P.-Zhao:2011wb}, $\ell_{sh}(c_i,c_j) = [1-(c_i - c_j)]_+^2$, 
\item and the rank-boost loss \cite{rudin_icml05}, $\ell_e(c_i,c_j) = e^{c_i-c_j}$. 
\end{enumerate}
All these surrogates are nonnegative, monotonic decreasing and satisfy $\ell(c_i,c_j) = 1$ when $c_i = c_j$.   One significant problem with these surrogates is that they all rely on pairwise comparisons between positive and negative training examples, and lead to algorithms with quadratic running time complexity.  For large datasets, such quadratic running time will significantly slow down the training process, and the pairwise comparisons prohibit efficient online learning algorithms for AUC optimization.

One exception is the work of \cite{ying_etal_nips16}, which shows that the squared hinge surrogate of AUC risk, $\ell_{sh}(c_i,c_j)$, affords an equivalent saddle point reformulation. An online stochastic gradient descent method is then developed based on this reformulation that has complexities $O(N)$ in time and $O(1)$ in space. However, there are two issues of this method that this work aims to improve on.  First, the original surrogate loss still requires pairwise comparison, and to decouple them, one needs to introduce auxiliary variables for a saddle point reformulation. In contrast, our surrogate loss obviates pairwise comparison all together. Second, our surrogate loss reduces to a minimization problem, which is easier to analyze and implement than the saddle point reformulation of  \cite{ying_etal_nips16}.

In parallel with methods directly optimizing AUC, empirical observations suggest that learning objectives not designed for AUC optimization (\eg, SVM or boosting) can achieve low AUC risk \cite{Brefeld:2005dd,Joachims:2005rv,Caruana:2004nr,Rakotomamonjy:2004oc}. For instance, in \cite{Joachims:2005rv}, a generalized SVM approach was developed that is able to optimize multivariate non-linear performance measures in polynomial time, including AUC. However, when assessed with respect to the AUC, the superiority of the direct AUC optimization approach over standard SVMs seemed less convincing. The work of \cite{Caruana:2004nr} many performance measures for binary classification are compared experimentally, and it was found that maximum margin methods such as boosting and SVMs yield excellent performance when measured with AUC. In \cite{Brefeld:2005dd} it was shown that optimizing standard SVMs leads to maximizing the AUC in the special (trivial) case when the given data is separable. As a perfect separation implies a zero AUC risk. The work \cite{steck_ecml07} uses the rank-equivalent definition of AUC to derive a hinge rank loss and shows that it is analogous to the SVM objective. However, no explicit relation between the SVM objective and AUC or AUC surrogates are established in previous works.

Further along this line, several studies have provided results on the consistency of the univariate losses to AUC risk, \ie, in the asymptotic sense, minimizing the univariate losses under certain conditions may also lead to the minimization of AUC risk \cite{clemencon_11,agarwal_13}, and a similar analysis is conducted for binary surrogate losses to AUC risk in \cite{gao_14}. These analyses show that univariate losses such as the $\ell_2$, squared hinge and exponential losses are consistent with AUC risk, yet the widely used hinge loss in SVM are inconsistent. This seems to put in question whether minimizing AUC risk based on pairwise comparisons is really warranted. However, these studies are still of limited in practice due to several reasons. First, they can not explain the observation that the SVM algorithm which is based on the hinge loss, oftentimes leads to good performance when evaluated with AUC risk, though it is not theoretically consistent with AUC risk. In addition, these analysis does not reveal a direct relation between the univariate losses and AUC risk, and it is more illustrative if some bounding relation between them can be revealed. Furthermore, these analyses may not be as relevant in practice, as the learning objective in actual algorithms is usually combined with extra terms such as the regularizers.  

\section{Method}

In this section, we start with an equivalent definition of AUC risk, which does not require pairwise comparisons of positive and negative examples. From this equivalent definition, we establish our AUC surrogate loss and its equivalent form for efficient optimization. 

\subsection{AUC Risk Without Pairwise Comparison}

Besides the WMW statistics, Eq.\eqref{e:auc}, there exists another equivalent formulation of AUC risk (and AUC itself), which  depends on the ranking of the prediction scores instead of all pairwise comparisons of the prediction scores of the positive and negative examples \cite{hand_till_ml01,steck_ecml07}.  To explain this equivalent form of AUC risk, we first introduce several additional notations. For simplicity, we assume there are no ties in the prediction scores, \ie, $c_i \not = c_j$ for $i \not = j$, though this condition will be relaxed later. 

We denote $(c^\uparrow_1,\cdots,c^\uparrow_N)$ as the result of sorting $(c_1,\cdots,c_N)$ in ascending order, \ie, $c^\uparrow_1 < c^\uparrow_2 < \cdots < c^\uparrow_N$.  Moreover, let $r^+_i  \in \{1,\cdots,N\}$ ($i = 1, \cdots, N^+$) be the {\em rank} of the i$^\text{th}$ positive example encountered in the ordered list $(c^\uparrow_1,\cdots,c^\uparrow_N)$ starting from the beginning. With a slight abuse of notation, let $c^{\uparrow+}_i$ be the corresponding value of the i$^\text{th}$ positive example in the ordered list $(c^\uparrow_1,\cdots,c^\uparrow_N)$, \ie, $c^{\uparrow+}_i = c^{\uparrow}_{r^+_i}$. An example illustrating these definitions is given in Fig.\ref{fig:1}.
\begin{figure}[t]
\begin{center}
\includegraphics[width=.48\textwidth]{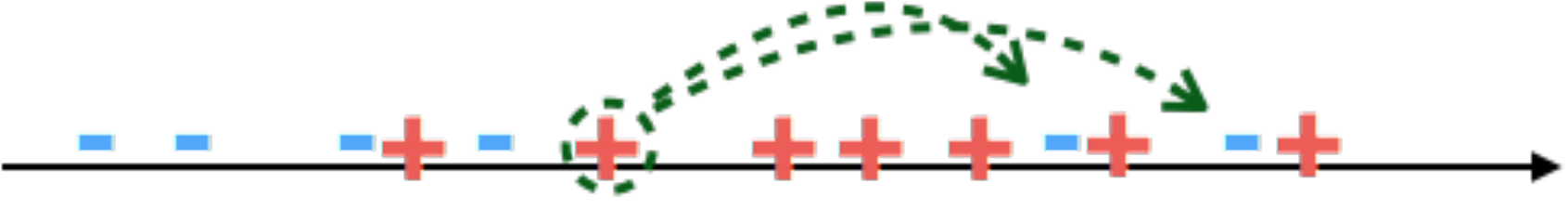}
\end{center}
~\vspace{-3em}
\caption{\small \em An illustration of the ranking definition of the AUC. Note that in this case, we have $N^+ = 7$, $N^- = 6$, and $(r^+_1,r^+_2,r^+_3,r^+_4,r^+_5,r^+_6,r^+_7) = (4,6,7,8,9,11,13)$. For the positive example highlighted with circle, it is the second positive example in the ordered list, and it is outranked by two negative examples (shown by arrows). So its contribution to AUC risk is $N^-+i - r^+_i = 6 + 2 - 6 = 2$. Repeating for all $7$ positive examples the total wrong pairs is $3 + 2 + 2 +2 +2 +1 +0 = 12$ and AUC risk is ${12 \over 6 \times 7} = {2 \over 7}$, which is the same as computed with Eq.\eqref{e:auc}.}
\label{fig:1}
~\vspace{-2em}
\end{figure}
These definitions immediately lead to the following simple result that will be important subsequently.
\begin{lemma} For $i = 1,\cdots, N^+$, we have $$r^+_i \le N^- + i,~~ c^\uparrow_{N^- + i} \ge c^{\uparrow+}_i.$$
\label{lem:1}
~\vspace{-2em} \end{lemma}
Proof of Lemma \ref{lem:1} is provided in the Appendix \ref{sec:app}. 

With these definitions, AUC risk can be defined using the rankings of the predictions \cite{hand_till_ml01}, which is equivalent to the definition based on the WMW statistics as given in Eq.\eqref{e:auc}. The intuition behind this equivalent form is a different way to count the number of reverse ordered pairs of positive and negative examples, which is illustrated with the numerical example in Figure \ref{fig:1}.
\begin{lemma}[\cite{hand_till_ml01}]
When there is no ties in training data, \ie, $c_i \not = c_j$ for $i \not = j$, we have
\begin{equation}
\begin{array}{ll}
\lauc & ={1 \over N^+N^-}  \sum_{i=1}^{N^+} (N^- + i- r^+_i) \\
&= 1 +  {N^+ + 1 \over 2N^-}- {1 \over N^+N^-}  \sum_{i=1}^{N^+}r^+_i.
\end{array}
\label{e:auc1}
\end{equation}
\label{lem:2}
~\vspace{-1em} \end{lemma}
Proof of Lemma \ref{lem:2} is provided in the Appendix \ref{sec:app}. Note that $\sum_{i=1}^{N^+} (N^- + i)$ corresponds to (trivially) the sum of the indices of the largest $N^+$ (top-$N^+$) elements in the ranked list of prediction scores, and $\sum_{i=1}^{N^+} r^+_i$ is the sum of the indices of positive examples in the ranked list of predictions. As such, AUC risk as defined in Eq.\eqref{e:auc1} is proportional to the difference between the two sums. 

This gives another intuitive explanation of AUC risk: in a perfect separable case, when the prediction scores of all the positive examples rank higher than those of the negative examples, \ie, all prediction scores of positive examples have ranks $N^- +1, \cdots, N$ in the ordered list, AUC risk is zero. In the more general cases, AUC risk measures how the rankings of the prediction scores deviate from this ideal case.

\subsection{Univariate Bound on AUC risk}

The equivalent form of AUC risk of Eq.\eqref{e:auc1} inspires a new surrogate loss based on the values of the sorted prediction scores $(c^\uparrow_1,\cdots, c^\uparrow_N)$.  To be specific, let us define a new quantity
\begin{equation}
\tilde{L} = {1 \over N^+N^-} \sum_{i = N^-+1}^N c^\uparrow_i  - {1 \over N^+N^-}\sum_{i \in \I^+} c_i.
\label{e:L1}
\end{equation}
Like AUC risk, $\tilde{L}$ is always nonnegative, as the second term, which is the sum of the prediction scores of all the positive examples, is less than or equal to the first term, which is the sum of the top-$N^+$ elements of $(c_1,\cdots, c_N)$. Equality holds only when the predictions of all positive examples rank higher than any of the negative examples. Our next result shows that we can bound AUC risk using $\tilde{L}$.
\begin{theorem} 
When there is no ties in training data, \ie, $c_i \not = c_j$ for $i \not = j$, we have $\tilde{L} \ge 0$.  Furthermore, there exist constants $\bar{\alpha} \ge \underline{\alpha} > 0$, such that $\bar{\alpha}\tilde{L} \ge \lauc \ge \underline{\alpha}\tilde{L}$. 
\label{lem:3}
~\vspace{-1em}
\end{theorem}
\begin{proof}
Using Lemma \ref{lem:1}, we have $\sum_{i=1}^{N^+} \lt(c^\uparrow_{N^- + i} - c^{\uparrow+}_i\rt) \ge 0$, therefore $\tilde{L} \ge 0$, and it is zero when $c^\uparrow_{N^- + i} = c^{\uparrow+}_i$ for all $i = 1, \cdots, N^+$, \ie, all positive examples outrank all negative examples.

We set $\bar{\alpha}^{-1} = \min_i (c^\uparrow_{i+1} - c^\uparrow_i) > 0$, and for $i > j$, we have $c^\uparrow_{i} - c^\uparrow_{j} = (c^\uparrow_{i} - c^\uparrow_{i-1}) + (c^\uparrow_{i-1} - c^\uparrow_{i-2}) + \cdots + (c^\uparrow_{j+1} - c^\uparrow_{j}) \ge {i-j \over \bar{\alpha}}$. %
Then we have
~\vspace{-.4em} \[ 
\begin{array}{ll}
\bar{\alpha} \tilde{L} &= {\bar{\alpha} \over N^+N^-}\sum_{i = N^-+1}^N c^\uparrow_i  - {\bar{\alpha} \over N^+N^-}\sum_{i \in \I^+} c_i \\
&= {\bar{\alpha} \over N^+N^-} \sum_{i=1}^{N^+} \lt(c^\uparrow_{N^- + i} - c^{\uparrow+}_i\rt) \\
& = {\bar{\alpha} \over N^+N^-} \sum_{i=1}^{N^+} \lt(c^\uparrow_{N^- + i} - c^{\uparrow}_{r^+_i}\rt) \\
&\ge {1 \over N^+N^-} \sum_{i=1}^{N^+} \lt(N^- + i - r^+_i\rt) = \lauc. 
 \end{array}
\]
Next, setting $\underline{\alpha}^{-1} = \max_i (c^\uparrow_{i+1} - c^\uparrow_i)$, and follow a similar derivation, we can obtain the other bound, \ie, $\lauc \ge \underline{\alpha}\tilde{L}$. The equalities in the bounds hold when $c^\uparrow_{i+1} - c^\uparrow_i$ is a constant for $i = 1,\cdots, N$, \ie, they are equally spaced. 
\end{proof}

\subsection{Computing $\tilde{L}$ without Explicit Ranking}

However, the ranking operation in $\tilde{L}$ is the main obstacle of using Eq.\eqref{e:L1} as a learning objective. However, this can be solved based on the following result on the sum of the top $k$ elements in a set \cite{Ogryczak:2003dl,fan_nips2017}, from which we can derive an equivalent form of Eq.\eqref{e:L1} that does not rely on ranking elements explicitly. 
\begin{lemma}[\cite{Ogryczak:2003dl,fan_nips2017}] 
For $N$ real numbers $z_1 < \ldots < z_N$, we have the equivalence of the sum-of-top-$k$ elements with an optimization problem as
\begin{equation}
\sum_{i=N-k+1}^{N} z_{i} = \min_{\lambda} \left\{ k\lambda + \sum_{i=1}^{N}\left[z_i - \lambda \right]_{+}  \right\},
\label{lt}
\end{equation}
with the optimal $\gl^\star$ satisfying $z_{N-k} \le \gl^\star < z_{N-k+1}$.
\label{lem:4}
~\vspace{-1em} \end{lemma}
Proof of Lemma \ref{lem:4} is provided in the Appendix \ref{sec:app}. Using Lemma \ref{lem:4}, we can rewrite $\tilde{L}$ by as a minimization problem over the auxiliary variable $\gl$, as
~\vspace{-.4em} \[ 
N^+N^-\tilde{L} = \min_{\lambda} \left\{ N^+\lambda + \sum_{i=1}^{N}\left[c_i - \lambda \right]_{+}   \right\} - \sum_{i \in \I^+} c_i,
\]
which can be further converted to
~\vspace{-.4em} \[
\min_{\lambda} \left\{\sum_{i \in \I^+} \lt(\left[c_i - \lambda \right]_{+} - (c_i - \lambda)\rt)+ \sum_{j \in \I^-} \left[c_j - \lambda \right]_{+}\right\}.
\]
Using the property of the hinge function that $[a]_+ - a = [-a]_+$, we can further simplify $\tilde{L}$, as
~\vspace{-.4em} \[ 
\begin{array}{ll}
\tilde{L} &= {1 \over N^+N^-} \min_{\lambda} \left\{  \sum_{i \in \I^+} \left[\lambda - c_i \right]_{+} + \sum_{j \in \I^-} \left[c_j - \lambda \right]_{+} \rt\} \\
&= {1 \over N^+N^-} \min_{\lambda} \sum_{i=1}^{N} \left[y_i(\lambda - c_i))\right]_{+}.
\end{array}
\]
Bringing back the parametric model to form a learning objective based on $\tilde{L}$ as
\begin{equation}
\tilde{L}(\w) = {1 \over N^+N^-} \min_{\lambda} \sum_{i=1}^{N} \left[y_i(\lambda - f_\w(\x_i))\right]_{+}.
\label{e:L2}
\end{equation}
This reformulation of $\tilde{L}$ is still a bound for AUC risk, but it does not require pairwise comparisons between predictions of positive and negative examples, and there is no need to explicitly ranking the predictions. Furthermore, in Eq.\eqref{e:L2}, the auxiliary variable $\lambda$ can be understood as a threshold that separates the two classes, and $\tilde{L}(\w)$ becomes independent of the choice of threshold by taking the overall minimum over all possible values for the threshold, as in the case of the original definition of AUC risk. 

The learning objective $\tilde{L}(\w)$ affords an intuitive interpretation in the context of binary classification.  It only penalizes those positive examples with predictions less than the threshold, \ie, $\left[\lambda - f_\w(\x_i)\right]_{+}$ for $i \in \I^+$, and negative examples with predictions greater than the threshold, \ie, $\left[f_\w(\x_i) - \lambda \right]_{+}$ for $i \in \I^-$. All examples that are on the ``correct'' side of the threshold receive no penalty. According to Lemma \ref{lem:4}, the optimal $\lambda$ takes value in the range of $[c^\uparrow_{N^+}, c^\uparrow_{N^++1})$.

\subsection{Relation with SVM Objective}
There are some strong similarities between $\tilde{L}(\w)$ and the SVM objective, which is particularly striking in the case of linear prediction function $f_\w(\x) = \w^\top\x$. This becomes clearer if we reformulate the SVM objective: if we regard the threshold $\lambda$ as the bias term in the linear prediction function for SVM\footnote{Typically in SVM we define the linear prediction function as $\w^\top\x+b$, but here we flip the sign of the bias so to better compare with $\tilde{L}(\w)$.}, $\w^\top\x - \lambda$, we can formulate the linear SVM objective \cite{cortes1995support} as 
~\vspace{-.4em} \[ 
\tilde{L}_\text{SVM}(\w,\lambda) = \sum_{i=1}^N [1 + y_i(\lambda - \w^\top\x_i)]_+.
\] 
Now comparing with Eq.\eqref{e:L2}, the two objectives has similar forms involving the hinge function.  We can further show that $\tilde{L}_\text{SVM}(\w,\lambda)$ is an upper-bound of $\tilde{L}(\w)$. This is because we have $[1 + y_i(\lambda - \w^\top\x_i)]_+ \ge [y_i(\lambda - \w^\top\x_i)]_+$, so
~\vspace{-.4em} \[ 
\begin{array}{ll}
\tilde{L}_\text{SVM}(\w,\lambda) &= \sum_{i=1}^N [1 + y_i(\lambda - \w^\top\x_i)]_+ \\
&\ge \sum_{i=1}^N [y_i(\lambda - \w^\top\x_i)]_+ \\
&\ge \min_\lambda \sum_{i=1}^N [y_i(\lambda - \w^\top\x_i)]_+ \\
&= \tilde{L}(\w).
\end{array}
\]
As we have shown in Theorem \ref{lem:3}, an upper-bound of AUC risk can be established with $\tilde{L}(\w)$, and this relation suggests the SVM objective $\tilde{L}_\text{SVM}(\w,\lambda)$ is also an upper-bound (albeit looser than $\tilde{L}$) of AUC risk.  

This helps to explain some long standing experimental observations (\eg, \cite{Brefeld:2005dd,Joachims:2005rv,Caruana:2004nr,Rakotomamonjy:2004oc}) that when assessed with AUC, standard SVMs could not be consistently outperformed by other approaches tailored to directly maximize AUC, such as RankBoost \cite{Freund:2003yy}, AUCsplit (local optimization of AUC) \cite{Herbrich:1999gl}, or ROC-SVM \cite{Rakotomamonjy:2004oc}.  

The two learning objectives also differ in two important aspects. The first is the constant $1$ in the SVM objective, which corresponds to the margin in constructing the binary classifier.  The second difference is that the bias $\lambda$ in $\tilde{L}$ is eliminated through minimization, but it is still present in the SVM objective.

\section{Optimization}
\label{s:25}

In this section, we discuss batch and online learning algorithms based on learning objectives formed from Eq.\eqref{e:L2}.

\subsection{Resolving Ties in Prediction Scores}
However, Eq.\eqref{e:L2} cannot be used as a learning objective due to one important issue. Note that in Eq.\eqref{e:L2}, the scale of the parameter $\w$ is not fixed, so the learning objective can be reduced by shrinking the scale of $\w$, which leads to a trivial solution with $\w = 0$. The underlying reason for this is that Eq.\eqref{e:L2} is based on the assumption of no ties in the  prediction scores, while the trivial solution corresponds to the extreme contrary, \ie, the prediction function always produce the same output (zero) regardless of the data. 

To resolve this problem, we augment the objective function with two other terms
\begin{equation}
\min_{\w} \tilde{L}(\w) + {\beta \over 2} \sum_{i=1}^{N} (1-y_i f_\w(\x_i))^2 + \gamma\Omega(\w),
\label{eq:1a}
\end{equation}
where the second term corresponds to a least squares term to counteract the effect of concentrating $\w$ to zero, the third term $\Omega(\w)$ is a regularizer on parameter $\w$, and $(\beta,\gamma)$ are weights to the two extra terms. 

\subsection{Linear Predictor}

In general, the learning objective of Eq.\eqref{eq:1a} is not a convex function of $\w$, but if we choose $f_\w(\x) = \w^\top\x$ and $\Omega(\w)$ is convex with respect to $\w$ (\ie, $\Omega(\w) = {1 \over 2} \|\w\|^2$), then we can show it is a convex function of $\w$. We first show that $[\x^\top\w - \lambda]_+$ is a convex function. For $\alpha \in [0,1]$, $\w$, $\w'$, $\lambda$, and $\gl'$, we have 
\begin{equation}
\begin{array}{ll}
& [\x^\top(\alpha \w + (1-\alpha)\w') - (\alpha \lambda+ (1-\alpha) \lambda')]_+ =\\
& [\alpha (\x^\top\w - \lambda)+  (1-\alpha)(\x^\top\w' - \lambda')]_+ \le \\
& \alpha[\x^\top\w - \lambda]_+ + (1-\alpha)[\x^\top\w' - \lambda']_+.
\end{array}
\label{e:cvx}
\end{equation}
Therefore, $\sum_{i=1}^{N}\left[\x^\top\w - \lambda \right]_{+}+ N^+\gl$ is a convex function jointly for $(\w,\gl)$. As the minimization of one variable in a joint convex function, $\min_\gl \sum_{i=1}^{N}\left[c_i - \lambda \right]_{+}+N^+\gl$ is also a convex function of $\w$. 

In summary, for the linear case, we can obtain the following convex learning objective with regards to $\w$ and $\gl$ jointly,
\begin{equation}
\begin{array}{ll}
&(\w^\star,\gl^\star) \leftarrow \argmin_{\w,\gl} {\gamma \over 2} \|\w\|^2 + \\
&\sum_{i=1}^{N} \lt\{\left[y_i(\lambda - \x_i^\top \w)\right]_{+} + {\beta \over 2} (1-y_i \x_i^\top \w)^2\rt\}
\end{array}
\label{eq:1}
\end{equation}
In the following, we discuss the batch and online optimization of Eq.\eqref{eq:1}, for which the convergence to global minimum is guaranteed.  

\subsubsection{Batch Learning}

In the batch setting, where we have access to all training examples, we can use block coordinate descent algorithm to optimize Eq.\eqref{eq:1}. We initialize $\w$ and $\lambda$, then iterate between
~\vspace{-.4em} \[ 
\begin{array}{ll}
\w^{(t+1)} &\leftarrow \argmin_\w \sum_{i=1}^N \left[y_i(\lambda^{(t)} - \w^\top \x_i)\right]_{+} + \\
& {\beta \over 2} \sum_{i=1}^N (1-y_i \x_i^\top \w)^2 + {\gamma \over 2} \|\w\|^2;\\
\gl^{(t+1)} &\leftarrow {1 \over 2} (c^\uparrow_{N^+} + c^\uparrow_{N^+ +1}),
\end{array}
\]
where $c^\uparrow_{i}$ is the rerank of $\{\x_i^\top \w^{(t+1)}\}_{i=1}^N$ in the ascending order. The $\w$ sub-problem can be converted to a constrained optimization problem as
~\vspace{-.4em} \[ 
\begin{array}{ll}
&\min_{\w,\vec{t}} \sum_{i=1}^N t_i +{\beta \over 2} \sum_{i=1}^N (1-y_i \x_i^\top \w)^2 + {\gamma \over 2} \|\w\|^2;\\
&\text{s.t.}~~ \y_i(\lambda^{(t)} - \w^\top \x_i) \ge t_i, t_i \ge 0.
\end{array}
\] 
This is a quadratic convex optimization problem and can be solved with interior point method when the dimensionality of $\w$ is low to medium. For high dimensional $\w$, the online learning algorithm is more effective as it avoids building the Hessian matrix.  

\subsubsection{Online Learning}

Because Eq.\eqref{eq:1} does not involve pairwise comparison, we can also derive an online learning algorithm based on stochastic gradient descent \cite{bousquet2008tradeoffs,srebro2010stochastic}. The runtime of the online algorithm does not depend on the number of training examples and thus this algorithm is especially suited for large datasets.  Specifically, with initial choice for the value of $\w^{(0)}$, at the $t^\text{th}$ iteration, a single training example $(\x_{i_t}, y_{i_t})$ is chosen at random from the training set and used to estimate a sub-gradient of the objective, and a step with pre-determined step-size is taken in the opposite direction, as
\begin{equation}
\begin{array}{ll}
\w^{(t+1)} & \leftarrow \w^{(t)} - \eta_t \lt((\gamma I + \beta \x_{i_t}^\top\x_{i_t}) \w^{(t)} - \rt. \\
& \lt. (\beta + \mathbf{I}_{y_i(\lambda^{(t)} - \w^\top \x_i)>0}) y_{i_t} \x_{i_t} \rt)\\
\gl^{(t+1)} & \leftarrow \gl^{(t)} - \eta_t y_{i_t} \mathbf{I}_{y_i(\lambda^{(t)} - \w^\top \x_i)>0},
\end{array}
\label{eq:2}
\end{equation}
where we can choose the step-size $\eta_t \sim {1\over \sqrt{t}}$, then the SGD algorithm will converge in $O(1/\epsilon)$ steps to the $\epsilon$-accuracy of the global optimal value of Eq.\eqref{e:L2} \cite{bousquet2008tradeoffs,srebro2010stochastic}. Note that each step of our online iterative algorithm has space and time complexity of $O(d)$ and $O(1)$, and obviates the need to store or buffer data in previous online AUC optimization methods \cite{W.-Gao:2013zo,P.-Zhao:2011wb}.


\section{Experiments}

We perform several experiments of learning binary classifiers to evaluate the batch and online algorithms optimizing learning objectives given in Eq.\eqref{eq:1} (subsequently denoted as {\tt ba-UBAUC} and {\tt ol-UBAUC}, respectively), and compare their performance with existing learning algorithms for AUC optimization. 

As in previous works \cite{W.-Gao:2013zo,ying_etal_nips16}, we perform experiments on $12$ benchmark datasets that have been used in previous studies. A summary of the data in these datasets is given in Table \ref{tab:1a}, with the training/testing split obtained from the original dataset. For datasets that are for data with more than $2$ class labels (\ie, {\tt news20} and {\tt sector}), following the convention of previous work \cite{W.-Gao:2013zo,ying_etal_nips16}, we convert them to binary classification problems by randomly partitioning the data into two groups, each with equal number of classes. Then the binary class labels are determined from the group to which the original class label belongs. Following the evaluation protocol of \cite{W.-Gao:2013zo,ying_etal_nips16}, the performance of reported is obtained by averaging the AUC scores on the test set for $25$ models learned from subsets of the same training set, each is chosen as a random $80\%$ of the original training data. 
\begin{table}[t]
  \centering
  \small{\begin{tabular}{|l|r|r|r|}
  \hline
 &train & test & data dim. \\
  \hline
  diabetes &389 & 389 &8 \\
  fourclass &431 & 431 &2 \\
  german &500 & 500 &24  \\
  splice &1,000& 2,175 &60 \\
   usps &7,291 &2,007&256  \\
   a9a &32,561 &16,281 &123 \\
  w8a & 49,749 &14,951& 300\\
    mnist &60,000 &10,000&780 \\
    acoustic &78,823 & 19,705 &50 \\
  ijcnn1 &49,990 &91,701&22  \\
   sector &6,412 &3,207&55,197 \\
    news20 &15,935 &3,993&62,061   \\
  \hline
  \end{tabular}
    }
    ~\vspace{-1em}
  \caption{\em \small Summary of the $12$ benchmark datasets used in our experiments. The training/testing splitting is from the original datasets.}
    \label{tab:1a} 
  ~\vspace{-2.5em}
\end{table}
\begin{table*}[t]
  \centering
  \setlength{\tabcolsep}{1.2pt}
  \footnotesize{
  \begin{tabular}{|c||c|c|c|c|c|c|c|c|}
  \hline
   & ol-UBAUC &ba-UBAUC &SOLAM &OPAUC &OAM$_{\textrm{seq}}$ &OAM$_{\textrm{gra}}$  &B-SVM-OR & SVM\\
  \hline
  diabetes  &.8326$\pm$.0299 &{ .8328}$\pm$.0352 &.8253$\pm$.0314 &.8309$\pm$.0350 &.8264$\pm$.0367 &.8262$\pm$.0338 &.8326$\pm$.0328 &.7821$\pm$.0145 \\
  fourclass &.8301$\pm$.0318 &{ .8310}$\pm$.0296 &.8226$\pm$.0240 &.8310$\pm$.0251 &.8306$\pm$.0247 &.8295$\pm$.0251 &.8305$\pm$.0311 &.7717$\pm$.0294 \\
  german   &.7928$\pm$.0371 & .7933$\pm$.0324 &.7882$\pm$.0243 &.7978$\pm$.0347 &.7747$\pm$.0411 &.7723$\pm$.0358 &{ .7935}$\pm$.0348 & .7641$\pm$.0283 \\
  splice    &.9231$\pm$.0224 &{ .9269}$\pm$.0094 &.9253$\pm$.0097 &.9232$\pm$.0099 &.8594$\pm$.0194 &.8864$\pm$.0166 & .9239$\pm$.0089 & .8439$\pm$.0096 \\
  usps      &.9728$\pm$.0051 &.9730$\pm$.0066 &{ .9766}$\pm$.0032 &.9620$\pm$.0040 &.9310$\pm$.0159 &.9348$\pm$.0122 &.9630$\pm$.0047 &.8930$\pm$.0075\\
  a9a       &.9005$\pm$.0019 &{ .9009}$\pm$.0041 &.9001$\pm$.0042 &.9002$\pm$.0047 &.8420$\pm$.0174 &.8571$\pm$.0173 &.9009$\pm$.0036 &.8213$\pm$.0064 \\
  w8a       &.9673$\pm$.0993 &{ .9695}$\pm$.0079 &.9114$\pm$.0075 &.9633$\pm$.0035 &.9304$\pm$.0074 &.9418$\pm$.0070 &.9495$\pm$.0082 &.8964$\pm$.0029\\
  mnist     &.9327$\pm$.0239 &{ .9340}$\pm$.0024 &.9324$\pm$.0020 &.9242$\pm$.0021 &.8615$\pm$.0087 &.8643$\pm$.0112 &.9340$\pm$.0020 &.8406$\pm$.0072\\
  acoustic  &.8871$\pm$.0035 &{ .8962}$\pm$.0046 &.8898$\pm$.0026 &.8192$\pm$.0032 &.7113$\pm$.0590 &.7711$\pm$.0217 &.8262$\pm$.0032 &.7629$\pm$.0045\\
  ijcnn1    &.9264$\pm$.0039 &{ .9337}$\pm$.0038 &.9215$\pm$.0045 &.9269$\pm$.0021 &.9209$\pm$.0079 &.9100$\pm$.0092 &.9337$\pm$.0024 &.8793$\pm$.0094\\
  sector    &{ .9845}$\pm$.0033 &- &.9834$\pm$.0023 &.9292$\pm$.0081 &.9163$\pm$.0087 &.9043$\pm$.0100                &- &.8815$\pm$.0062 \\
  news20    &{ .9468}$\pm$.0045 &-&.9467$\pm$.0039 &.8871$\pm$.0083 &.8543$\pm$.0099 &.8346$\pm$.0094                &- &.8431$\pm$.0127\\
  \hline
  \end{tabular}}
  ~\vspace{-0.5em}
    \caption{ \small \em Comparison of the AUC scores (mean$\pm$std.) on test sets of the evaluated datasets.}
      \label{tab:2} 
   ~\vspace{-2em}
\end{table*}

On these datasets, we evaluate and compare UBAUC-based algorithms with four state-of-the-art online and two batch learning algorithms for learning linear binary classifiers that minimizes various pairwise surrogates to the original AUC risk $\lauc$. The hyper parameters $(\beta,\gamma)$ for UBAUC are determined by a grid search on the validation set.  The initial learning rate for the online learning algorithm is also set for different dataset by a grid search. We compare the following algorithms with UBAUC-based algorithms.
\begin{itemize} \itemsep 0em
\item SOLAM  \cite{ying_etal_nips16}, an online AUC optimization algorithm based on a saddle point reformulation of the pairwise $\ell_2$ surrogate loss of AUC risk; 
\item OPAUC \cite{P.-Zhao:2011wb}, an online AUC optimization algorithm that uses the pairwise  $\ell_2$ loss surrogate of the AUC objective function; 
\item OAM \cite{W.-Gao:2013zo}, an online AUC optimization algorithm that uses the pairwise  hinge loss surrogate of the AUC objective function with two variants, one with sequential update (OAMseq) and the other using gradient update (OAMgra);
\item B-SVM-OR \cite{Joachims:2005rv}, a batch learning algorithm using the pairwise hinge loss surrogate of the AUC objective function;
\item UNI-SVM, a linear SVM algorithm implemented using LIBSVM with SMO minimization \cite{CC01a}. 
\end{itemize}
Classification performances measured by the AUC score on the testing dataset of all compared methods for all $12$ benchmark datasets are given in Table \ref{tab:2}. For fair comparison, we implement all algorithms using MATLAB, and following the default parameter settings in the original papers. Note that the simple implementation of the two batch algorithms cannot handle datasets with high dimensional datasets, \ie, {\tt sector} and {\tt news20}, due to the memory requirement. However, for those datasets that it is feasible to run, {\tt ba-UBAUC}, the batch version optimizing the proposed learning objective, performs best.  On the other hand, the results of u{\tt UNI-SVM}, though optimizing a different objective, still achieves reasonable performance when evaluated with AUC. The online algorithm based on the proposed learning objective, {\tt ol-UBAUC}, achieves comparable performance as other state-of-the-art online algorithms based on pairwise surrogate losses to AUC risk, although the improvements of performance on some of the datasets are not conspicuous due to the nature of the data. 

On the other hand, the main advantage of {\tt ol-UBAUC} in comparison with other online algorithms is the running efficiency -- its per-iteration running time and space complexity is linear in data dimension and do not depend on the iteration number. Furthermore, each iteration of {\tt ol-UBAUC} Eq. \eqref{eq:2} corresponds to a simpler update step than the saddle point solve in SOLAM \cite{ying_etal_nips16}. In Table \ref{tab:3}, we show the per-iteration running time and the total running time for the learning objective function to converge to have smaller than $10^{-7}$ relative changes\footnote{Experiments were performed with running time reported based on a cluster with 12 nodes, each with an Intel Xeon E5-2620 2.0GHz CPU and 64GB RAM. All algorithms are implemented using MATLAB, with available code obtained from the authors of the corresponding publications.} of the five online algorithms we compared. Note that the online version of the UBAUC-based algorithms has more efficient running time with comparable performances in comparison to existing AUC optimization methods.
\begin{table}[t]
\begin{center}
\begin{tabular}{|c|r|r|r|}
\hline
&{\small   {\tt a9a} }&{\small   {\tt usps} }&{\small   {\tt sector}}\\
\hline
{\scriptsize  ol-UBAUC  }&{\scriptsize 0.48}&{\scriptsize   0.15 }&{\scriptsize   11.24}\\
\hline
{\scriptsize  SOLAM }&{\scriptsize 0.50}&{\scriptsize   0.19 }&{\scriptsize   19.90}\\
\hline
{\scriptsize  OPAUC }&{\scriptsize 6.24}&{\scriptsize   4.62}&{\scriptsize   120.30}\\
\hline
{\scriptsize  OAMseq }&{\scriptsize 34.31}&{\scriptsize   13.98}&{\scriptsize   1350.41}\\
\hline
{\scriptsize  OAMgra }&{\scriptsize 34.35}&{\scriptsize   12.54}&{\scriptsize   1350.50}\\
\hline
\end{tabular}
\begin{tabular}{|c|r|r|r|}
\hline
&{\small   {\tt a9a} }&{\small   {\tt usps} }&{\small   {\tt sector}}\\
\hline
{\scriptsize  ol-UBAUC  }&{\scriptsize 0.83 }&{\scriptsize   0.15/0.58 }&{\scriptsize  276.41}\\
\hline
{\scriptsize  SOLAM }&{\scriptsize 0.99 }&{\scriptsize   0.19/0.81 }&{\scriptsize   721.52}\\
\hline
{\scriptsize  OPAUC }&{\scriptsize 14.21 }&{\scriptsize   4.62/11.23}&{\scriptsize   5540.24}\\
\hline
{\scriptsize  OAMseq }&{\scriptsize 78.42 }&{\scriptsize   13.98/32.71}&{\scriptsize   6730.75}\\
\hline
{\scriptsize  OAMgra }&{\scriptsize 69.23 }&{\scriptsize   12.54/39.54}&{\scriptsize   6324.64}\\
\hline
\end{tabular}
\end{center}
   ~\vspace{-1.5em}
\caption{\em \small {\bf (top)}The average running time (in seconds) per pass over training data  for each online algorithm,  and {\bf (bottom)} the average running time (in seconds) for the learning objective function to converge to have smaller than $10^{-7}$ relative changes for each online algorithm.}
\label{tab:3}
   ~\vspace{-3em}
\end{table}%

\section{Population Form}

So far, we have described the proposed learning objective over a set of finite training data. In this section, we discuss the population form of the surrogate loss using probability distributions of data. This analysis will shed light on the formal connection of the new surrogate loss with existing methods and can lead to deeper theoretical studies.

We start with the population form of the equivalent definition of AUC risk in Eq.\eqref{e:auc1}. We assume that the input data and label are from a joint model $p(\x,y)$, which induces density models for the predictions $c = f(\x)$. As such, we denote $\rho^+(c) = p(c|y=1)$ and $\rho^-(c) = p(c|y=-1)$ as the (conditional) probability density functions (PDFs) for positive and negative class, respectively. For simplicity, we assume both PDFs have infinite support, \ie, is non-zero for the whole $\R$. Also, we denote $p = Pr(y=1)$ as the class prior probability. 

The joint probability density function of the classification output $c$ is then given by $\rho(c) = p \rho^+(c) + (1-p) \rho^-(c)$. We also denote $F^+(c) = \int_{-\infty}^c \rho^+(c')dc'$, $F^-(c) = \int_{-\infty}^c \rho^-(c')dc'$, and $F(c) = \int_{-\infty}^c \rho(c')dc'$ as the cumulative distribution functions (CDFs) for $\rho^+$, $\rho^-$ and $\rho$, respectively, with $F(c) = pF^+(c) + (1-p)F^-(c)$. $F^+(c)$ is the false negative rate (FNR) and $1-F^-(c)$ is the false positive rate (FPR). 

AUC risk is defined as the area under the whole curve of FNR vs. FPR, as $\lauc = \int_{-\infty}^{\infty} (1-F^-(c)) dF^+(c)$ \cite{hand_till_ml01}. Using relation $F^-(c) = {1 \over 1-p}(F(c) - pF^+(c))$ yields
~\vspace{-.4em} \[ 
\lauc = {1 \over 1-p} \int_{-\infty}^{\infty} (1-p + pF^+(c) - F(c)) dF^+(c).
\]
Because $F$ is a CDF is a continuous monotonic function and $F(c) \le 1-p + pF^+(c) \le 1$, using the mean value theorem, there exists $c' \ge c_0 = \max\{c|F(c) = 1-p\}$, such that $1-p = F(c_0) \le F(c') = 1-p + pF^+(c) \le 1$, and
~\vspace{-.4em} \[ 
\lauc = {1 \over 1-p} \int_{-\infty}^{\infty} (F(c') - F(c)) dF^+(c).
\]
Next, note that $F(c)$ is Lipschitz with constant $\alpha' \ge \max_c |\rho(c)|$, \ie, $|F(c') - F(c)| \le \alpha' |c'-c|$, we have
\begin{equation}
\lauc \le {\alpha' \over 1-p} \int_{-\infty}^{\infty} (c' - c) dF^+(c). 
\label{e:aux}
\end{equation}
Next, we use the following result 
\begin{lemma}
For $F(c') = 1-p + pF^+(c)$, we have 
~\vspace{-.4em} \[ 
\int_{-\infty}^{\infty}c' dF^+(c) = \min_\lambda  \int_{-\infty}^{\infty} {(c-\lambda)_+ \over p} dF(c) + \lambda.
\] 
\label{lem:4a}
~\vspace{-1em} \end{lemma}
Proof of Lemma \ref{lem:4a} is provided in the Appendix \ref{sec:app}. Using Lemma \ref{lem:4a}, we can rewrite the integral of the right hand side of Eq.\eqref{e:aux} as
~\vspace{-.4em} \[ 
\min_\lambda  \int_{-\infty}^{\infty} {(c-\lambda)_+ \over p} dF(c) + \lambda - \int_{-\infty}^{\infty}c dF^+(c),
\]
where the terms being minimized can be further simplified as
~\vspace{-.4em} \[ 
\int_{-\infty}^{\infty} {(c-\lambda)_+ \over p} dF(c) + (\lambda - c) dF^+(c).
\]
This can be further expanded using the relation $dF(c) = (1-p) dF^-(c) + pdF^+(c)$ to have
~\vspace{-.4em} \[ 
\begin{array}{ll}
&\int_{-\infty}^{\infty} (c-\lambda)_+ (1-p) dF^{-}(c) + \\
&\int_{-\infty}^{\infty} \left[(\lambda - c) + (c-\lambda)_+\right] p dF^+(c).
\end{array}
\]
Putting all terms together and using the relation $(c-\lambda)_++(\lambda-c) = (\lambda-c)_+$ we have
~\vspace{-.4em} \begin{equation}
\lauc \le {\alpha' \over p(1-p)} \min_\lambda E_{c,y} [y(c-\lambda)]_+,
\label{e:L3}
\end{equation}
where $E_{c,y}$ represents the expectation over $c$ and $y$.

\section{Conclusion}

In this work, we describe a new surrogate loss to the AUC metric based on a formulation of AUC, which does not require pairwise comparison but rankings of the prediction scores. We further show that the ranking operation can be avoided and the learning objective obtained based on this surrogate affords complexity in time and storage that is linear in the number of training data. We perform experiments to demonstrate the effectiveness of the online and batch algorithms for AUC optimization based on the proposed surrogate. 

There are several directions we would like to further explore for this work. First, form the theoretical point of view, we would like to establish the consistency of the proposed learning objective with regards to AUC risk, \ie, the question if the surrogate loss will also lead to the convergence to the optimal AUC risk. The form of our surrogate loss (Eq.\eqref{e:L2}) as an optimization problems makes it difficult to apply the techniques used in previous works \cite{clemencon_11,agarwal_13} to this case. We would also like to establish the generalization error between the data form of the loss Eq.\eqref{e:L2} and its population form counterpart Eq.\eqref{e:L3}. From the algorithm perspective, we would like to extend this learning objective to substitute multi-class AUC \cite{hand_till_ml01}, where multi-class AUC risk is computed as the average of binary class AUC between each pairs of classes. Last, we are interested in applying the online algorithm based on the proposed surrogate loss to non-convex learning objectives such as those used for training deep neural networks. 

\textbf{Acknowledgement}. Siwei Lyu is supported by the National Science Foundation (NSF, Grant IIS-1537257) and Yiming Ying is supported by the Simons Foundation (\#422504) and the 2016-2017 Presidential Innovation Fund for Research and Scholarship (PIFRS) program from SUNY Albany.

\appendix
\section{Appendix: Proofs}
\label{sec:app}

\begin{proof}[Proof of Lemma \ref{lem:1}]
Being the i$^\text{th}$ positive example encountered in the ordered list $(c^\uparrow_1,\cdots,c^\uparrow_N)$, $c^{\uparrow+}_i$ can outrank no more than $N^- + i$ elements in the list, \ie, $i$ positive examples and at most $N^-$ negative examples.  Therefore, we have $r^+_i \le N^- + i$. By the ranking order we also have $c^\uparrow_{N^- + i} \ge c^{\uparrow}_{r^+_i} = c^{\uparrow+}_i$.
\end{proof}

\begin{proof}[Proof of Lemma \ref{lem:2}]
Consider the $i^\text{th}$ positive example encountered in $(c^\uparrow_1,\cdots,c^\uparrow_N)$ starting from the beginning, which has rank $r^+_i$. The number of negative examples that rank lower than it is $r^+_i-i$, \ie, there will be $N^- - (r^+_i - i) = N^- + i- r^+_i$ negative examples with ranks higher than this positive example, \ie, forming a reversed ordered pair with it. This corresponds to the sum over reversed ordered pairs in the definition of AUC risks of Eq.\eqref{e:auc}. Summing over all such reverse ordered pairs divided by the number of all such positive-negative pairs ($N^+N^-$) proves the result.
\end{proof}

\begin{proof}[Proof of Lemma \ref{lem:4}]
First, we note that $\sum_{i=N-k+1}^{N} z_{i}$ is the solution of the following linear programming problem
\begin{equation}
\max_{\bp \in R^{n\times 1}} \ \bp^\top \vec{z}, \ \ \st \ \ \bp^\top \bone = k, \, p_i \in [0,1],
\label{prime_0}
\end{equation}
We form its Lagrangian as 
\begin{equation}
L = -\bp^\top \bz - \mathbf{a}^\top \bp + \mathbf{b}^\top(\bp - \bone) + \lambda (\bp^\top \bone - k),
\label{lagrange}
\end{equation}
where $\mathbf{a} \geq \mathbf{0}$, $\mathbf{b} \geq \mathbf{0}$ and $\lambda$ are Lagrangian multipliers. Setting the derivative of $L$ with respect to $\bp$ to be $\bzero$, we obtain
$\mathbf{a} = \mathbf{b} -\vec{z} +\lambda \bone$. Substituting this into Eq \eqref{lagrange}
, we get the dual problem of (\ref{prime_0}) as
\begin{equation}
\min_{\mathbf{b},\lambda} \ \mathbf{b}^\top \vec{1} + k\lambda, \ \st \ \mathbf{b} \geq \bzero, \mathbf{b} + \lambda \bone - \bz \geq \bzero,
\label{dual0}
\end{equation}
The constraints of Eq. \eqref{dual0} suggest that we should have $\mathbf{b}^\top \vec{1} \ge \sum_{i=1}^{n}\left[z_i - \lambda \right]_{+}$. As such, the objective function achieves its minimum when the equality holds. Reorganizing terms leads to Eq.\eqref{lt}. Further, when we choose $\gl^\star$ satisfying $z_{N-k} \le \gl^\star < z_{N-k+1}$, we have $k\lambda^\star + \sum_{i=1}^{N}\left[z_i - \lambda^\star \right]_{+} = k\lambda^\star + \sum_{i=N-k+1}^{N}(z_i - \lambda^\star) = \sum_{i=N-k+1}^{N} z_{i}$. Thus proves the lemma.   
\end{proof}

\begin{proof}[Proof of Lemma \ref{lem:4a}]
First, we have $dF(c') = pdF^+(c)$, then $\int_{-\infty}^{\infty}c' dF^+(c) = {1 \over p} \int_{c_0}^{\infty}c' dF(c')$, where the lower limit of the integral, $c_0 = \max\{c|F(c) = 1-p\}$, originates from the range of value $c'$. Next, we compute $\min_\lambda  \int_{-\infty}^{\infty} (c-\lambda)_+ dF(c) + p \lambda = \min_\lambda  \int_{\lambda}^{\infty} c dF(c) - \lambda \int_{\lambda}^{\infty} dF(c) + p \lambda$. Differentiating the inner terms with regards to $\lambda$,  we obtain  $\int_{\lambda}^{\infty} dF(c) = p$, or $\int_{-\infty}^{\lambda} dF(c) = 1-p$, so we have at optimum, $\lambda = c_0$. Therefore we have $\min_\lambda  \int_{-\infty}^{\infty} (c-\lambda)_+ dF(c) + p \lambda = \int_{c_0}^{\infty}c' dF(c')$. Further rearranging terms proves the result.
\end{proof}
{
\bibliography{refs}
\bibliographystyle{ieeetr}
}

\end{document}

%% file: mydefs.tex
\usepackage{graphicx}
\usepackage{subfigure}
\usepackage{wrapfig}
\usepackage{amsmath,amsfonts,amsthm,mathrsfs}
\usepackage{array}
\usepackage{multirow}
\usepackage{bbm}
\usepackage{color}
\usepackage{algorithm} 
\usepackage{algorithmic} 
\usepackage{epstopdf}
\usepackage{footmisc}

\newcommand{\argmin}{\mathop{ \arg\!\min}}

\renewcommand{\vec}[1]{\mathbf{#1}}

\usepackage{url}
\usepackage{epsfig}
\usepackage{subfigure}
\usepackage{graphicx}
\usepackage{multirow}
\usepackage{paralist}
\usepackage{color}

\def\lt{\left}
\def\rt{\right}

\def\x{\vec{x}}
\def\y{\vec{y}}
\def\w{\vec{w}}
\def\R{{\cal R}}

\def\bz{\mathcal Z}

\def\bp{\mathbf{p}}
\def\bone{\mathbf{1}}
\def\bzero{\mathbf{0}}
\def\sgn{\hbox{sign}}

\def\eg{\emph{e.g.}} 
\def\ie{\emph{i.e.}}

\def\st{\text{s.t.}}

\def\R{\mathbb{R}}

\def\I{\mathcal{I}}

\def\gl{\lambda}
\def\R{\mathbb{R}}

\newtheorem{theorem}{Theorem}

\newtheorem{lemma}{Lemma}

\allowdisplaybreaks[4]